\documentclass[a4paper]{article}

\usepackage[preprint, nonatbib]{nips_2018}

\usepackage[english]{babel}
\usepackage[utf8x]{inputenc}
\usepackage[T1]{fontenc}
\usepackage[noend]{algpseudocode}
\usepackage{fancyhdr,amssymb}
\usepackage[ruled,vlined]{algorithm2e}
\usepackage{array}
\usepackage{stmaryrd}
\usepackage{amsfonts}

\usepackage{amsfonts,amsmath,amssymb,amsthm,epsfig,psfrag}
\usepackage[T1]{fontenc}
\usepackage{babel}

\usepackage{amsmath}
\usepackage{mathtools}
\usepackage{dsfont}
\usepackage{systeme}
\usepackage{graphicx}
\usepackage[colorinlistoftodos]{todonotes}

\usepackage{blindtext}
\usepackage{enumitem}
\usepackage{xcolor}
\newtheorem{lemma}{Lemma}

\newtheorem{proposition}[lemma]{Proposition}

\newcommand{\RR}{{\mathbb{R}}}
\newcommand{\KL}{ \mathrm{KL}}
\newcommand{\BEAS}{\begin{eqnarray*}}
\newcommand{\EEAS}{\end{eqnarray*}}
\newcommand{\BEA}{\begin{eqnarray}}
\newcommand{\EEA}{\end{eqnarray}}
\newcommand{\BEQ}{\begin{equation}}
\newcommand{\EEQ}{\end{equation}}
\newcommand{\BIT}{\begin{itemize}}
\newcommand{\EIT}{\end{itemize}}
\newcommand{\BNUM}{\begin{enumerate}}
\newcommand{\ENUM}{\end{enumerate}}
\setlist[description]{leftmargin=\parindent,labelindent=\parindent}

\title{InfoCatVAE: Representation Learning\\ with Categorical Variational Autoencoders}

\author{
  Marc Lelarge \\
  \texttt{INRIA - ENS}\\
  \texttt{marc.lelarge@ens.fr}
  \And
  Edouard Pineau \\
  \texttt{Telecom ParisTech - Safran}\\
  \texttt{edouard.pineau@safrangroup.com}
}

\begin{document}
\maketitle

\begin{abstract}

This paper describes InfoCatVAE, an extension of the variational autoencoder that enables unsupervised disentangled representation learning. InfoCatVAE uses multimodal distributions for the prior and the inference network and then maximizes the evidence lower bound objective (ELBO). We connect the new ELBO derived for our model with a natural soft clustering objective which explains the robustness of our approach. We then adapt the InfoGANs method to our setting in order to maximize the mutual information between the categorical code and the generated inputs and obtain an improved model
\end{abstract}

\section{Introduction}

Neural networks are today a state of the art solution for many machine learning tasks. In particular, they show impressive results in specific tasks that require capturing complex features in data. Nevertheless, most of the neutral network success is associated with globally unavailable large labelled data-sets. To be more generally applicable, machine learning needs new unsupervised methods to leverage the largely available unlabelled training data. One role of unsupervised learning is capturing rich distributions of complex data. More specifically among unsupervised tasks, representation learning seek to expose semantically meaningful factors hidden in data. Recently generative neural network models have become a highly successful framework for this problem. In particular, variational autoencoders (VAEs) \cite{kingma2013auto} and generative adversarial networks (GANs) \cite{goodfellow2014generative} are major representation learning frameworks.

To be fully relevant for representation learning, unsupervised generative models have to encode the observable data in an informative space \cite{bengio2013representation}. For example, VAEs are traditionally built with isotropic Gaussian latent distribution that theoretically should learn one specific and exclusive latent semantically meaningful factor of variation per dimension. This idea is part of the generic definition of the $disentanglement$, assuming that any structured visible object has a factorial representation \cite{ridgeway2016survey}. Recently, several extensions of VAEs with richer prior distributions have been proposed to obtain more powerful generative representation models \cite{tomczak2017vae,dilokthanakul2016deep,nalisnick2016approximate,van2014factoring,goyal2017nonparametric}. As shown in \cite{hoffman2016elbo} one main difficulty associated to these models is the choice of the prior. 

We propose an extension of the standard VAE framework with multimodal distributions for the prior as well as for the inference network. We derive the  evidence lower bound objective (ELBO) for this new model with categorical prior to define Categorical VAE (CatVAE). We show that thanks to these modifications, even with a simple fixed prior, CatVAE is able to find salient attributes within the data leading to readable representations. Moreover, CatVAE can be used as a conditional generative model. 
Finally, we present an improved extension of the model to enhance disentanglement and the quality of the generated samples. More generally, this work aims at showing that simple and principled modifications can have interesting disentangling power without the need of specific heuristics or highly fine-tuned neural networks architectures.

\section{Related work}
\label{related}

Variational autoencoding is today the core framework in which disentanglement is investigated \cite{zhao2017infovae,higgins2016beta,chen2016variational,makhzani2015adversarial,esmaeili2018hierarchical,chen2018isolating}. A major drawback known in VAE framework comes from the VAE objective function, the evidence lower bound (ELBO), that has a form inducing uninformative and unstable representation power. 

Among the numerous studies that have been proposed about disentangled representation learning, a majority is based on a regularization along already existing objective functions for different generative models \cite{makhzani2015adversarial,zhao2017infovae,higgins2016beta,chen2016infogan,gao2018auto,kim2018disentangling,burgess2018understanding,esmaeili2018hierarchical,chen2018isolating}. Other studies focus on the clustering aspect of disentangling latent representation by exploring the effect of mixture distribution for latent representation in variational autoencoders \cite{tomczak2017vae,dilokthanakul2016deep,nalisnick2016approximate,van2014factoring}. The mixture of distributions has the advantage of empowering the latent distribution to represent more complex variations in the data. Moreover, in term of disentangling, a mixture implicitly assumes that there is a hierarchical importance among meaningful factors of variations: each mode represents a major latent factor and local isotropic distributions hold minor disentangling power. Nevertheless, the mixture models hold several limits that depends on the chosen form of the prior and the form of the objective function. 

Our work follows the footsteps of the Mixture of Gaussian based models presented in \cite{nalisnick2016approximate,dilokthanakul2016deep}. We propose not to use free prior parameterization via neural network transformation but to fix the prior parameters like in standard VAE framework (section \ref{sec:priorchoice}) but we also modify the inference network and rederive the ELBO of our model. Finally we improve our model with explicit information maximization.

\section{Background: Variational Autoencoders}

Consider a latent variable model with a data variable $x\in \mathcal{X}$ and a latent variable $z\in \mathcal{Z}$, $p(z,x) = p(z)p_\theta(x|z)$. Given the data $x_1,\dots, x_n$, we want to train the model by maximizing the marginal log-likelihood:
\BEA
\label{eq:L}
\mathcal{L} = \mathbb{E}_{p_d(x)}\left[\log p_\theta(x)\right]=\mathbb{E}_{p_d(x)}\left[\log \int_{\mathcal{Z}}p_{\theta}(x|z)p(z)dz\right],
  \EEA
  where $p_d$ denotes the empirical distribution of $X$: $p_d(x) =\frac{1}{n}\sum_{i=1}^n \delta_{x_i}(x)$.

  To avoid the (often) difficult computation of the integral in (\ref{eq:L}), the idea behind variational methods is to instead maximize a lower bound to the log-likelihood (called ELBO):
  \BEA\label{eq:var}
\mathcal{L} \geq L(p_\theta(x|z),q(z|x)) =\mathbb{E}_{p_d(x)}\left[\mathbb{E}_{q(z|x)}\left[\log p_\theta(x|z)\right]-\KL\left( q(z|x)||p(z)\right)\right].
  \EEA
  Any choice of $q(z|x)$ gives a valid lower bound. Variational autoencoders replace the variational posterior $q(z|x)$ by an inference network $q_{\phi}(z|x)$ that is trained together with $p_{\theta}(x|z)$ to jointly maximize $L(p_\theta,q_\phi)$. The variational posterior $q_{\phi}(z|x)$ is also called the encoder and the generative model $p_{\theta}(x|z)$, the decoder or generator.

  It will be convenient to decompose the optimization done in the VAEs in two steps:
  \BEA
  \label{eq:gen} \max_{\theta, \phi}\mathbb{E}_{p_d(x)}\left[\mathbb{E}_{q_\phi(z|x)}\left[\log p_\theta(x|z)\right]\right]\\
  \label{eq:inf} \max_{\phi} \mathbb{E}_{p_d(x)}\left[-\KL\left( q_\phi(z|x)||p(z)\right)\right]
  \EEA

  The term in (\ref{eq:gen}) is the negative reconstruction error. Indeed under a gaussian assumption i.e. $p_{\theta}(x|z) = \mathcal{N}(\mu_{\theta}(z), 1)$ the term $\log p_\theta(x|z)$ reduced to $\propto \|x-\mu_\theta(z)\|^2$, which is often used in practice. The maximization (\ref{eq:inf}) can be seen as a regularization term, where the variational posterior $q_\phi(z|x)$ should be matched to the prior $p(z)$.

  \section{Modifying VAEs for representation learning}
  \label{sec:catvae}

 As already noted in the literature, the VAE objective is insufficient for representation learning. This is particularly true in the so-called high capacity regime, where $p_{\theta}(x|z) = p_d(x)$ is close to achievable. In this case, a strategy for the VAE is for the encoder to match the prior $q_\phi(z|x) = p(z)$, while the decoder outputs the sample distribution without using the latent code $z$. In such a case, the latent code $z$ is independent from the data, hence useless for representation theory. 

  This is clearly due to the fact that in (\ref{eq:L}), only the marginal $p_\theta(x)$ appears and the only way VAEs enforce information in the latent code representation is by limiting the optimization in (\ref{eq:gen}) to a constrained class of decoders $p_\theta(x|z)$. In \cite{higgins2016beta}, to solve this problem, the authors modify the ELBO by putting more weight to the term (\ref{eq:inf}) through the introduction of a new parameter $\beta$. This approach is extended in \cite{dupont2018jointvae} where this parameter is modified during the training (like an annealing procedure).

We propose to reuse the base block of semi-supervised method found in \cite{kingma2014semi} to enforce information in the latent code representation. We call it CatVAE (for Categorical VAE). This architecture will be improved in the next section to obtain InfoCatVAE. CatVAE consists in the three following modifications to the VAE:

\BIT
\item We modify the latent variable model as follows:
  $p(x,c,z) = p(c)p(z|c)p_\theta(x|z)$ where $c\in\{1,\dots K\}$ is a discrete latent variable. In other words, the prior $p(z)$ of the VAEs is replaced by a prior $p(c)p(z|c)$. For instance, if the data are images from the MNIST dataset, then $c$ would encode the numerical identity of the digit (0-9) and then similarly as in standard VAEs, we would take a prior $p(z|c) =\mathcal{N}(z;\mu_c,Id)$ for a well-chosen $\mu_c$.
\item We define the inference network by $q_\phi(c|x)q_{\phi}(z|x,c)$.
\item We modify the objective (\ref{eq:inf}) as follows:
  \BEA
  \label{eq:regucatvae}
\max_\phi \mathbb{E}_{p_d(x)}\left[-\KL\left( q_\phi(c|x)||p(c)\right)\right] - \mathbb{E}_{p_d(x)}\left[ \mathbb{E}_{q_\phi(c|x)} \KL\left(q_\phi(z|c,x)||p(z|c) \right)\right].
  \EEA
\EIT

\begin{figure}[ht]
  \begin{center}
  \includegraphics[width=1\textwidth]{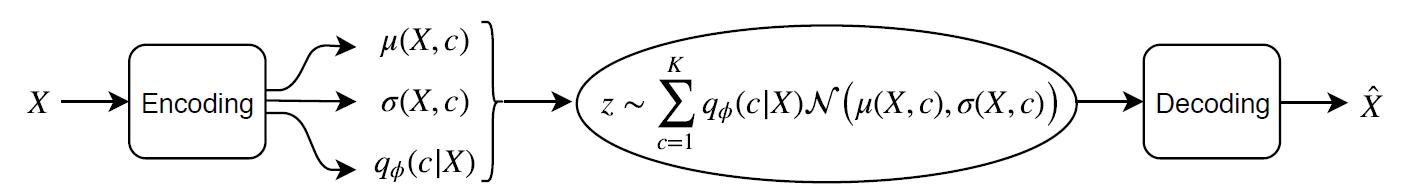}
  \end{center}
\caption{\label{fig:cat_vae} CatVAE: square blocks represent neural networks, oval-shaped blocks represent sampling.}
\end{figure}  

A representation of CatVAE architecture is illustrated in figure \ref{fig:cat_vae}.
Note that our proposition is not a VAE with a prior given by a mixture of Gaussians but still CatVAE maximizes a variational lower bound of the marginal log-likelihood as shown by the following proposition:

  \begin{proposition}
We have $\mathcal{L} \geq (3) + (5)$.
  \end{proposition}
  \begin{proof}
    See appendix \ref{app2}
    \end{proof}

Also CatVAE still optimizes the (log-)marginal $p_\theta(x)$, it enforces information in the latent code thanks to the fixed prior $p(c)p(z|c)$ which should be chosen appropriately, see Section \ref{sec:priorchoice}. The modifications proposed for CatVAE raise some difficulties due to the categorical variable $c$ (it is well known that the reparametrization trick which is used for VAEs cannot be directly applied to discrete variables). We will explain how we overcome these difficulties in Section \ref{sec:cat_opt}.
  
  In order to have more intuition about the term (\ref{eq:regucatvae}), we consider the case where $p(c)$ is the uniform distribution over $\{1,\dots, K\}$, $q_\phi(z|c,x)= \mathcal{N}(z;\mu_\phi(c,x),\sigma_{\phi}^2(c,x))$ and $p(z \vert c) = \mathcal{N}(z;\mu_c,1)$.
  In this case, the first term is 
  \BEAS
-\KL\left( q_\phi(c|x)||p(c)\right) = \log K + H\left(q_\phi(c|x)\right),
\EEAS
hence maximizing it is the same as maximizing the entropy of the distributin $q_\phi(c|x)$, i.e. each category should be evenly represented. Then the second term becomes

\BEAS
-\mathbb{E}_{q_\phi(c|x)} \KL\left(q_\phi(z|c,x)||p(z|c) \right) =\frac{1}{2} \sum_{c=1}^K q_\phi(c|x)\left(1-\sigma_{\phi}^2(c,x)+\log \sigma_{\phi}^2(c,x)-(\mu_\phi(c,x) - \mu_c)^2 \right)
\EEAS

For a given $c$, if $q_\phi(c|x)$ is close to one, i.e. if the category associated to $x$ is with high probability the $c$'s category, then the variance $\sigma_{\phi}^2(c,x)$ should be close to one and the mean $\mu_\phi(c,x)$ close to $\mu_\ell$, in order to maximize this term.
In summary, the optimization step (\ref{eq:regucatvae}) can be interpreted as a soft clustering step in $\RR^d$ where $d$ is the dimension of the code $z$ and with $K$ clusters. Each data point $x$ is mapped to $\RR^d$ thanks to the function $\mu_\phi(c,x)$. The centers of the clusters are given by the means $\mu_c$ for $c\in \{1,\dots, K\}$, the classifier $q(c|x)$ is then a soft allocation of the data $x$ to the cluster $c$ and the functions $\mu_\phi$ are updated to get closer to the center of the cluster. The entropy term prevents the trivial solution where all points are mapped to one cluster and ensures that data points are evenly distributed among the $K$ different clusters. Note that step (\ref{eq:regucatvae}) tends to concentrate each cluster and this step will be mitigated by the reconstruction step (\ref{eq:gen}). It is also important to note a main difference comparing to the other approaches in the literature described below. To the best of our knowledge, our CatVAE is the only architecture where all the distances between the code and all clusters are explicitly computed and used in the backpropagation algorithm. This is in contrast with standard approaches where typically first a classifier determines the cluster and then only the distance between the code and the selected cluster is computed. We believe that this specificity in our approach explains the robustness of CatVAE. Of course, our approach will be problematic if the number of clusters is very large.

Note that once the CatVAE has been trained, it can be used as a generative model: a data point of the category $c$ can be generated by first sampling a code with the distribution $p(z|c)$ and then passing it through the encoder $p_\theta(x|z)$. The center of the cluster $c$ is mapped to $p_\theta(x|\mu_c)$.

\section{InfoCatVAE: CatVAE with information maximization}\label{sec:acat}

We noted that in InfoGANs, the classifier is trained on generated inputs only and this already provide impressive results. In our CatVAE model, the classifier is only trained on real data. But since our CatVAE can be turned into a generative model, we can also use the generated inputs to improve our classifier. By analogy with InfoGAN, this can be easily done by modifying the step (\ref{eq:regucatvae}) as follows:


 \BEA
\max_\phi \mathbb{E}_{p_d(x)}\left[-\KL\left( q_\phi(c|x)||p(c)\right)\right] - \mathbb{E}_{p_d(x)}\left[ \mathbb{E}_{q_\phi(c|x)} \KL\left(q_\phi(z|c,x)||p(z|c) \right)\right] \\ \label{eq:reguinfocatvae} + {\mathbb{E}_{p(c)p(z \vert c)}\Big[ \mathbb{E}_{p_{\theta}(x \vert z)} \big[ \log q_{\phi}(c \vert x) \big] \Big]}
  \EEA
  
We call this model the InfoCatVAE. Indeed the same computation as in \cite{chen2016infogan} can be carried out here and the additional term (\ref{eq:reguinfocatvae}) is a variational lower bound of mutual information between the category $c$ and the generated inupt $p_\theta(x|z)$ (where $z$ has been sampled with distribution given by the prior $p(z|c)$)

We find illustrative experiments of the advantage of using information maximization in Section \ref{sec:exp}. The additional part of InfoCatVAE over CatVAE is illustrated in figure \ref{fig:cat_vae_arch}.

\begin{figure}[ht]
  \begin{center}
  \includegraphics[width=1\textwidth]{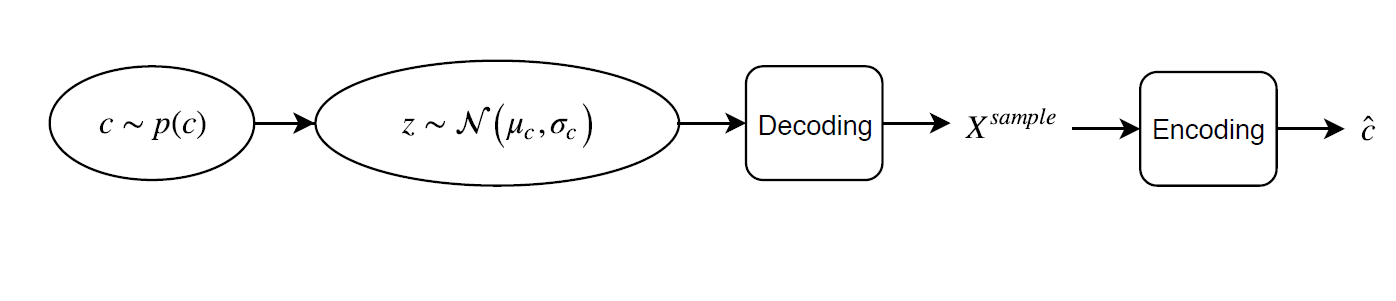}
  \end{center}
\caption{\label{fig:cat_vae_arch} InfoCatVAE: square blocks represent neural networks, oval-shaped blocks represent sampling. This brick is added to CatVAE architecture and trained at the same time. Encoding and decoding blocks are shared with CatVAE presented in figure \ref{fig:cat_vae}.}
\end{figure}

In practice, the three terms in (\ref{eq:reguinfocatvae}) are respectively multiplied by scalar factor $\beta_{cont}$, $\beta_{cat}$ and $\beta_{info}$ \cite{higgins2016beta}. These new parameters act as trade-off between reconstruction, latent distribution matching and preservation of salient information within latent space.

\section{Optimization with categorical sampling layer}
\label{sec:cat_opt}

The reparametrization trick developed in \cite{kingma2013auto} has no natural equivalence for multinomial sampling. A method called Gumbel max trick \cite{gumbel2012statistics} is widely used in machine learning to overcome this problem \cite{jang2016categorical,dupont2018jointvae}. In InfoCatVAE learning, we propose to overpass this problem not by using Gumbel max trick but with an alternative two-step method that is naturally induced by our model.

First, as seen above, in the inference network, $q_\phi (c \vert x)$ for all $c$'s and all data points $x$. Each $x$ is then represented $K$ times in the latent space, and each representation is weighted by the probability $q_\phi (c \vert x)$. This first step enables back-propagation by still keeping the spirit of a uniform discrete sampling conditioned on the data.

Second, the information maximization step presented in section \ref{sec:acat} optimizes the same bricks than inference learning, but with a free generative approach. This second step then propose a framework to optimize the InfoCatVAE network with a real categorical sampling, that does not block the back-propagation since the sampling is the initial layer. 

Therefore, by construction, infoCatVAE naturally enables categorical sampling optimization.

\section{Choice of the prior}
\label{sec:priorchoice}

As presented in Section \ref{sec:catvae}, CatVAE requires the choice of the parameters $\mu _c$ of the prior distribution $p(z \vert c)$. Our intuition is that the prior should be fixed such that it fills the same objective than the isotropic structure of standard VAE. We remind that the isotropic Gaussian latent distribution should theoretically learn one specific and exclusive latent semantically meaningful factor of variation per dimension. 

We choose the dimensions of the latent space $d$ such that $\exists \ \delta \in \mathbb{N} \ s.t. \ d=K.\delta$. We consider that the data with $K$ categories should be encoded with a $K$-modal distribution approximated with $K$ Gaussian whose mean parameters $\{\mu_c\}_{c=1}^K$ lives in $\mathbb{R}^{d}$ and are all respectively orthogonal. De facto $\forall c \in \llbracket {1:K} \rrbracket$ we propose:

\begin{equation}
\mu_c=\{\lambda.\mathds{1}_{j \in \llbracket {c\times \delta:(c+1) \times \delta} \llbracket} \}_{j=1}^{d}
\end{equation}

The main idea behind this choice of prior is that each major categories within the data should be mainly represented within a $\delta-$dimensional subspace of the latent space. This framework forces the model to learn quasi-independently each fundamental class structure. This way, the interpretability of the latent representation should be optimized.

  \section{Relation with InfoGANs}

Generative adversarial networks (GANs) train generative models through an objective function that implements a two-player zero sum game between a discriminator $D$ and a generator $G$. That is $G$ maps random vectors $z$ to generated inputs $\tilde{x} =G(z)$ and we assume $D$ to predict the probability of example $x$ being present in the dataset: $p(y=1|x,D)= (1+e^{-D(x)})^{-1}$. But GANs share with standard VAEs the absence of restrictions on the manner generator should use the noise. This way, there are no insurance that latent representation would be disentangled.

InfoGANs propose to decompose the input noise vector into two parts: $z$ which is treated as source of incompressible noise and $c$ which will represent salient semantic features of the data distribution. InfoGANs also introduce a variational posterior $q(c|x)$. To highlight the similarities between our framework and InfoGANs, we use parameter $\theta$ for the generator $G_\theta(c,z)$ and for the variational posterior $q_\theta(c|x)$ and parameter $\phi$ for the discriminator $D_\phi(x)$. With these notations, the minimax game with a variational regularization of mutual information and hyperparameter $\lambda$ solved by InfoGANS can be written as:
  \BEA
\label{eq:gan}  \max_{\theta} \mathbb{E}_{p(z)}\log p\left( y=1|G_{\theta}(c,z), D_\phi\right)+\lambda \mathbb{E}_{p(z)p(c)}\log q_\theta(c|G_\theta(c,z))\\
\label{eq:disc}  \max_{\phi} \mathbb{E}_{p_d(x)}\log p(y=1|x,D_\phi) - \mathbb{E}_{p(z)} \log p(y=1|G_\theta(c,z),D_\phi).
  \EEA
  Step (\ref{eq:gan}) updates the generator $G_\theta$ as well as the classifier $q_\theta$ while step (\ref{eq:disc}) updates the discriminator $D_{\phi}$.

 We refer to \cite{hu2017unifying} for an in-depth comparison of VAEs and GANs. As explained above, VAEs can be seen as generative models and similarly GANs can produce embeddings from the data. Indeed, the discriminator is a neural network with one final fully connected layer to output the boolean parameter. A natural encoder is then given by the discriminator with the last fully connected layer removed. Note however that the discriminator is trained to detect generated sample from real data so that the features that will be kept by the discriminator are those helping into the discimination task. It is not clear a priori that those features will be the most readable one as expected in representation learning. In practice, for InfoGANs, $D$ and $q$ share all layers except the last one, so that the encoder described above for GANs will still work for InfoGANs and keep informations about the categories.

  Note also that the classifier of the InfoGAN $q_\theta$ is trained only on generated inputs, whereas the classifier of the CatVAE is trained on real data but with 'noisy' labels. We will present in Section \ref{sec:acat} an extension of CatVAE building on this remark in order to improve the performance of the classifier.

  \section{Relation with adversarial autoencoders}
  
  Adversarial autoencoders (AAEs) build on standard (i.e. non variational) autoencoders with a deterministic encoding function denoted here $E_\eta(x)$ mapping each input $x$ to a code $z$ and a generative process $p_\theta(x|z)$. The regularization in AAEs is done thanks to an adversarial network matching the prior $p(z)$ we want to impose on the code with the aggregated posterior distribution $ p_d(E_\eta(x)) $. We denote by $D_\phi$ the discriminator of the adversarial network so that updates can now be written as:
  \BEA
\label{eq:aae1}  \max_{\theta, \eta}\mathbb{E}_{p_d(x)} \log p_{\theta}\left( x|E_\eta(x)\right) + \mathbb{E}_{p_d(x)} \log p\left(y=1|E_\eta(x), D_\phi \right).\\
\label{eq:aae2}  \max_\phi \mathbb{E}_{p(z)}\log p\left( y=1|z, D_\phi\right) - \mathbb{E}_{p_d(x)} \log p\left(y=1|E_\eta(x), D_\phi \right).
\EEA

Step (\ref{eq:aae1}) updates the encoder and the decoder in order to minimize the reconstruction loss (first term in (\ref{eq:aae1})), as well as the encoder in order to confuse the discriminator (second term in (\ref{eq:aae1})).
Step (\ref{eq:aae2}) is the classical update of the discriminator.

As in the modification from GAN to InfoGAN, structure can be imposed onto the prior by using a code $(c,z)$ with a distribution $p(c)p(z)$ and typically $p(c)$ is a categorical distribution. Also the encoder $E_\eta$ now generates both a category $c$ and a continuous code $z$. We see that the first component of $E_\eta(x)$ plays the role of a classifier which can be used for unsupervised clustering.

Also the losses are not the same as ours, AAEs is very similar to our CatVAE. However, we note that AAEs inforce only a certain type of prior on the code $(c,z)$. Indeed, the category $c$ is one-hot encoded which corresponds in our CatVAE to orthogonals means $(\mu_c, c\in \{1,\dots,K\})$ for the distributions $p(z|c)$ of our CatVAE. Our CatVAE allows us to be more flexible for the choice of the priors.


\section{Experiment setup and illustrative results}\label{sec:exp}

In this section, we aim at illustrating that InfoCatVAE enables readable representation and controlled generation. Therefore, we first illustrate our work with MNIST and FashionMNIST data with trivial multilayer perceptron architecture (see table \ref{tab:table1} in appendix \ref{app1})

Figure \ref{fig:interpolation} and figure \ref{fig:dim_extension} illustrate the accomplishment of the InfoCatVAE on readable representation task.

\begin{figure}[ht]
  \begin{center}
  \includegraphics[width=0.3\textwidth]{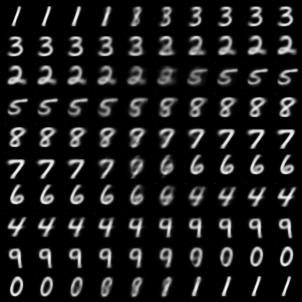}
  \includegraphics[width=0.3\textwidth]{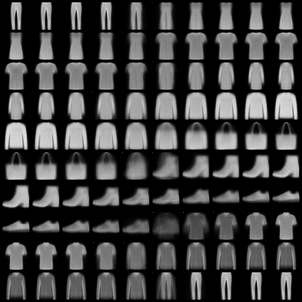}
  \end{center}
\caption{\label{fig:interpolation} Discrete interpolation between the prior centroïds $\mu_c$ when InfoCatVAE is trained with $K=10$ and $\lambda=2$ respectively on MNIST (left) and FashionMNIST (right). The left columns show the ten centroids represented in the observable space. Each line represent the reconstruction of nine latent values that pad the path between the ten centroids.}
\end{figure}

\begin{figure}[ht]
  \begin{center}
  \includegraphics[width=0.3\textwidth]{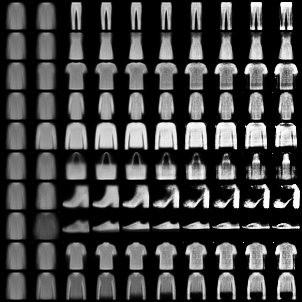}
  \end{center}
\caption{\label{fig:dim_extension} Impact of a variation of the $\lambda$ term of the prior mean parameter from $0$ to $9$ when InfoCatVAE is trained with $K=10$ and $\lambda=2$ on FashionMNIST. We see that this variation has an emptying impact on every generated samples. Nevertheless, not all have human readable explanation. For example, it makes sense for the transformation of shoes into sandal or shoulder sleeves to shoulder straps dress, but not for trousers. }
\end{figure}

We reimplemented the Adversarial Autoencoder \cite{makhzani2015adversarial} with the same multilayer perceptron encoder and decoder as presented in appendix \ref{app1} (table \ref{tab:table1}) to compare our sampling capacity with comparable state-of-the-art variational autoencoder sampling.

\begin{table}[h!]
  \begin{center}
    \begin{tabular}{l|l}
      \textbf{Model} & \textbf{MNIST LL score}\\
      \hline
      Adversarial autoencoder $(K=1)$ & 95.5 \\
      CatVAE $(K=10)$ & 111.2 \\
      InfoCatVAE $(K=10)$ & 113.8 \\
    \end{tabular}
  \end{center}
  \caption{Log-likelihood of the 10K generated samples from different generative models trained on MNIST with 600 epochs. Higher values are better. The density function is estimated by the KernelDensity function of scikit-learn \cite{pedregosa2011scikit}, whose bandwidth parameter has been estimated via grid search with 5 folds cross-validation over the 60K training examples. Each model has the same encoder and decoder architecture than presented in appendix \ref{app1}.}
   \label{tab:table2}
\end{table}

The table \ref{tab:table2} illustrates the fact that CatVAE and InfoCatVAE are more adapted for sampling task than AAE when architecture parameters are reduced to their simplest form, despite the necessity for CatVAE and InfoCatVAE to learn disentangled representation.

Finally, like in InfoGAN, (\ref{eq:reguinfocatvae}) can be approximated via Monte Carlo simulation. We generate 10K samples from 10K discrete labels sampled from multinomial distribution and we compute the cross-entropy labels and inferred classes of the samples. This framework enables easy optimization of the mutual information by error gradient back-propagation. In term of result, CatVAE we find a cross-entropy of 2.03 and for InfoCatVAE a cross-entropy of 1.62. Therefore, mutual information between generated samples and categories is improved with InfoCatVAE.

\section{Robustness of the InfoCatVAE in high-capacity regime}

This section aim at showing that our choice of prior brings robustness in high capacity regime. 

Our work proposes in addition to generative information maximization the fixing of a trivial multimodal prior distribution (see section \ref{sec:catvae} and \ref{sec:priorchoice}). This framework stands between multimodal free prior learning framework \cite{nalisnick2016approximate,dilokthanakul2016deep} and unimodal fixed prior \cite{dupont2018jointvae}. In particular, we have the intuition that for high-capacity regime tasks, relaxing the fixed prior makes the learning unstable.

An illustrative task is the representation learning of multivariate time series using recurrent encoder and recurrent decoder. The particularity of sequential data representation is the necessity to encode the temporal dimension in a non-temporal space. In the particular case of unsupervised representation learning, the complexity of the objective task associated to the gaps of recurrent modeling generally demands more exotic models to achieve acceptable representative and generative power \cite{boulanger2012modeling,graves2013generating,bowman2015generating}. The rich multi-scaled structure of sequential data naturally induces the necessity to get a hierarchical representation. The readable hierarchical approach for sequential data has been treated recently \cite{chung2016hierarchical,hsu2017unsupervised,li2018deep, liang2017recurrent}. All these particularities of sequential data gives an example of complex task that can unveil the limits of models that work well simple static data representation. 

For the illustrative experiment, we continue using MNIST data but by considering each image as a 28-dimensional time series of length 28. This way the trivial structure of shared by similar digits collapses. To do so, we simply take our InfoCatVAE model with fixed prior (as described in section \ref{sec:priorchoice}) and recurrent encoder and decoder, and compare it to the same model with free priors as in \cite{dilokthanakul2016deep}. 

The figure \ref{fig:rec_mnist} illustrates the kind of result that we obtain for multivariate sequential MNIST. It shows that the association of fixed multimodal prior framework and information generation maximization bring robustness and interpretability power to our model. The details of the implementation are given in appendix \ref{app3}.

\begin{figure}[h]
  \begin{center}
  \includegraphics[width=0.3\textwidth]{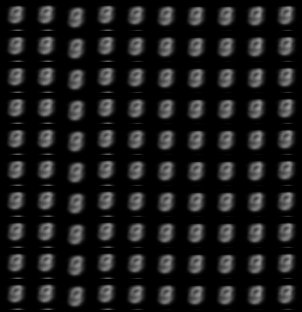}
  \includegraphics[width=0.31\textwidth]{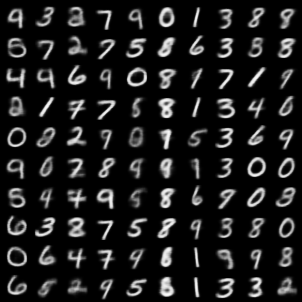}
  \includegraphics[width=0.31\textwidth]{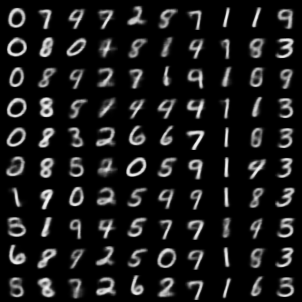}
  \end{center}
\caption{\label{fig:rec_mnist} (left) Samples generated from free prior recurrent CatVAE. (middle) Samples generated from recurrent CatVAE. (right) Samples generated from recurrent InfoCatVAE. Results obtained with CatVAE and InfoCatVAE confirm that our models learn categorical salient attributes within the data and structure the latent space such that those attributes are disentangled, even when problem is hard.}
\end{figure}

\section{Conclusion and future work}

In this paper, we have introduced the CatVAE, a multimodal variational autoencoder with fixed multimodal prior. We have shown that this model can learn disentangled representation of data and a highly credible conditional generation framework. Moreover, we show that contrary to more complex and flexible models, CatVAE overpasses complex task without the need of specifically fine-tuned architectures. Finally, we show that we can extend the CatVAE to an information optimized version, the InfoCatVAE. This framework both enhance generation and categorical information learning. 

Despite encouraging results, we largely accept that fixing the prior can pose information transfer and representation problems. The free prior framework has multiple concepts and for specific tasks, a well controlled free prior with a fine-tuned architecture might be more powerful then InfoCatVAE. But for a robust generalization, our method seems more appropriate. 

As a future work, we could prefix the prior parameters with bayesian hyper parametrization \cite{nalisnick2016approximate,nalisnick2016deep}. This way, we would not have to completely free the prior parameters through neural nets during the learning and therefore keep the stability of the InfoCatVAE, while still improving the structure of the latent space.

\bibliographystyle{plain}

\begin{thebibliography}{10}

\bibitem{bengio2013representation}
Yoshua Bengio, Aaron Courville, and Pascal Vincent.
\newblock Representation learning: A review and new perspectives.
\newblock {\em IEEE transactions on pattern analysis and machine intelligence},
  35(8):1798--1828, 2013.

\bibitem{boulanger2012modeling}
Nicolas Boulanger-Lewandowski, Yoshua Bengio, and Pascal Vincent.
\newblock Modeling temporal dependencies in high-dimensional sequences:
  Application to polyphonic music generation and transcription.
\newblock {\em arXiv preprint arXiv:1206.6392}, 2012.

\bibitem{bowman2015generating}
Samuel~R Bowman, Luke Vilnis, Oriol Vinyals, Andrew~M Dai, Rafal Jozefowicz,
  and Samy Bengio.
\newblock Generating sentences from a continuous space.
\newblock {\em arXiv preprint arXiv:1511.06349}, 2015.

\bibitem{burgess2018understanding}
Christopher~P Burgess, Irina Higgins, Arka Pal, Loic Matthey, Nick Watters,
  Guillaume Desjardins, and Alexander Lerchner.
\newblock Understanding disentangling in beta-vae.
\newblock {\em arXiv preprint arXiv:1804.03599}, 2018.

\bibitem{chen2018isolating}
Tian~Qi Chen, Xuechen Li, Roger Grosse, and David Duvenaud.
\newblock Isolating sources of disentanglement in variational autoencoders.
\newblock {\em arXiv preprint arXiv:1802.04942}, 2018.

\bibitem{chen2016infogan}
Xi~Chen, Yan Duan, Rein Houthooft, John Schulman, Ilya Sutskever, and Pieter
  Abbeel.
\newblock Infogan: Interpretable representation learning by information
  maximizing generative adversarial nets.
\newblock In {\em Advances in Neural Information Processing Systems}, pages
  2172--2180, 2016.

\bibitem{chen2016variational}
Xi~Chen, Diederik~P Kingma, Tim Salimans, Yan Duan, Prafulla Dhariwal, John
  Schulman, Ilya Sutskever, and Pieter Abbeel.
\newblock Variational lossy autoencoder.
\newblock {\em arXiv preprint arXiv:1611.02731}, 2016.

\bibitem{chung2016hierarchical}
Junyoung Chung, Sungjin Ahn, and Yoshua Bengio.
\newblock Hierarchical multiscale recurrent neural networks.
\newblock {\em arXiv preprint arXiv:1609.01704}, 2016.

\bibitem{dilokthanakul2016deep}
Nat Dilokthanakul, Pedro~AM Mediano, Marta Garnelo, Matthew~CH Lee, Hugh
  Salimbeni, Kai Arulkumaran, and Murray Shanahan.
\newblock Deep unsupervised clustering with gaussian mixture variational
  autoencoders.
\newblock {\em arXiv preprint arXiv:1611.02648}, 2016.

\bibitem{dupont2018jointvae}
Emilien Dupont.
\newblock Learning disentangled joint continuous and discrete representations.
\newblock {\em arXiv preprint arXiv:1804.00104}, 2018.

\bibitem{esmaeili2018hierarchical}
Babak Esmaeili, Hao Wu, Sarthak Jain, N.~Siddharth, Brooks Paige, and
  Jan-Willem Van~de Meent.
\newblock Hierarchical disentangled representations.
\newblock {\em arXiv preprint arXiv:1804.02086v2}, 2018.

\bibitem{gao2018auto}
Shuyang Gao, Rob Brekelmans, Greg~Ver Steeg, and Aram Galstyan.
\newblock Auto-encoding total correlation explanation.
\newblock {\em arXiv preprint arXiv:1802.05822}, 2018.

\bibitem{goodfellow2014generative}
Ian Goodfellow, Jean Pouget-Abadie, Mehdi Mirza, Bing Xu, David Warde-Farley,
  Sherjil Ozair, Aaron Courville, and Yoshua Bengio.
\newblock Generative adversarial nets.
\newblock In {\em Advances in neural information processing systems}, pages
  2672--2680, 2014.

\bibitem{goyal2017nonparametric}
Prasoon Goyal, Zhiting Hu, Xiaodan Liang, Chenyu Wang, and Eric Xing.
\newblock Nonparametric variational auto-encoders for hierarchical
  representation learning.
\newblock {\em arXiv preprint arXiv:1703.07027}, 2017.

\bibitem{graves2013generating}
Alex Graves.
\newblock Generating sequences with recurrent neural networks.
\newblock {\em arXiv preprint arXiv:1308.0850}, 2013.

\bibitem{gumbel2012statistics}
Emil~Julius Gumbel.
\newblock {\em Statistics of extremes}.
\newblock Courier Corporation, 2012.

\bibitem{higgins2016beta}
Irina Higgins, Loic Matthey, Arka Pal, Christopher Burgess, Xavier Glorot,
  Matthew Botvinick, Shakir Mohamed, and Alexander Lerchner.
\newblock beta-vae: Learning basic visual concepts with a constrained
  variational framework.
\newblock 2016.

\bibitem{hoffman2016elbo}
Matthew~D Hoffman and Matthew~J Johnson.
\newblock Elbo surgery: yet another way to carve up the variational evidence
  lower bound.
\newblock In {\em Workshop in Advances in Approximate Bayesian Inference,
  NIPS}, 2016.

\bibitem{hsu2017unsupervised}
Wei-Ning Hsu, Yu~Zhang, and James Glass.
\newblock Unsupervised learning of disentangled and interpretable
  representations from sequential data.
\newblock In {\em Advances in neural information processing systems}, pages
  1876--1887, 2017.

\bibitem{hu2017unifying}
Zhiting Hu, Zichao Yang, Ruslan Salakhutdinov, and Eric~P Xing.
\newblock On unifying deep generative models.
\newblock {\em arXiv preprint arXiv:1706.00550}, 2017.

\bibitem{jang2016categorical}
Eric Jang, Shixiang Gu, and Ben Poole.
\newblock Categorical reparameterization with gumbel-softmax.
\newblock {\em arXiv preprint arXiv:1611.01144}, 2016.

\bibitem{kim2018disentangling}
Hyunjik Kim and Andriy Mnih.
\newblock Disentangling by factorising.
\newblock {\em arXiv preprint arXiv:1802.05983}, 2018.

\bibitem{kingma2013auto}
Diederik~P Kingma and Max Welling.
\newblock Auto-encoding variational bayes.
\newblock {\em arXiv preprint arXiv:1312.6114}, 2013.

\bibitem{kingma2014semi}
Diederik~P Kingma and Max Welling.
\newblock Semi-supervised learning with deep generative models.
\newblock In {\em Advances in neural information processing systems}, pages
  3581--3589, 2014.

\bibitem{li2018deep}
Yingzhen Li and Stephan Mandt.
\newblock A deep generative model for disentangled representations of
  sequential data.
\newblock {\em arXiv preprint arXiv:1803.02991}, 2018.

\bibitem{liang2017recurrent}
Xiaodan Liang, Zhiting Hu, Hao Zhang, Chuang Gan, and Eric~P Xing.
\newblock Recurrent topic-transition gan for visual paragraph generation.
\newblock {\em CoRR, abs/1703.07022}, 2, 2017.

\bibitem{makhzani2015adversarial}
Alireza Makhzani, Jonathon Shlens, Navdeep Jaitly, Ian Goodfellow, and Brendan
  Frey.
\newblock Adversarial autoencoders.
\newblock {\em arXiv preprint arXiv:1511.05644}, 2015.

\bibitem{nalisnick2016approximate}
Eric Nalisnick, Lars Hertel, and Padhraic Smyth.
\newblock Approximate inference for deep latent gaussian mixtures.
\newblock In {\em NIPS Workshop on Bayesian Deep Learning}, volume~2, 2016.

\bibitem{nalisnick2016deep}
Eric Nalisnick and Padhraic Smyth.
\newblock Deep generative models with stick-breaking priors.
\newblock {\em arXiv preprint arXiv:1605.06197}, 2016.

\bibitem{pedregosa2011scikit}
Fabian Pedregosa, Gael Varoquaux, Alexandre Gramfort, Vincent Michel, Bertrand
  Thirion, Olivier Grisel, Mathieu Blondel, Peter Prettenhofer, Ron Weiss,
  Vincent Dubourg, et~al.
\newblock Scikit-learn machine learning in python.
\newblock {\em Journal of machine learning research}, 12(Oct):2825--2830, 2011.

\bibitem{ridgeway2016survey}
Karl Ridgeway.
\newblock A survey of inductive biases for factorial representation-learning.
\newblock {\em arXiv preprint arXiv:1612.05299}, 2016.

\bibitem{tomczak2017vae}
Jakub~M Tomczak and Max Welling.
\newblock Vae with a vampprior.
\newblock {\em arXiv preprint arXiv:1705.07120}, 2017.

\bibitem{van2014factoring}
Aaron Van~den Oord and Benjamin Schrauwen.
\newblock Factoring variations in natural images with deep gaussian mixture
  models.
\newblock In {\em Advances in Neural Information Processing Systems}, pages
  3518--3526, 2014.

\bibitem{zhao2017infovae}
Shengjia Zhao, Jiaming Song, and Stefano Ermon.
\newblock Infovae: Information maximizing variational autoencoders.
\newblock {\em arXiv preprint arXiv:1706.02262}, 2017.

\end{thebibliography}

\clearpage
\appendix

\section{Proof of proposition 1}
\label{app2}

\begin{align*}
    &\log p_{\theta} (x)  &&= \log \int _{\mathcal{Z}}{\sum_{\ell=1}^K{p_{\theta} (x,z,c=\ell)}}dz \\
    & &&= \log \int _{\mathcal{Z}}{\sum_{\ell=1}^K{p_{\theta}(x|z,c=\ell)p_{\theta}(z|c)p_{\theta}(c)} \frac{q_{\phi} (z,c|x)}{q_{\phi} (z,c|x)}}dz \\
    & &&= \log \mathbb{E}_{z,c \sim q_{\phi} (z,c|x)}{\Big[\frac{p_{\theta}(x|z,c)p_{\theta}(z|c)p_{\theta}(c)}{q_{\phi} (z,c|x)} \Big]} \\
    & && \geq \mathbb{E}_{z,c \sim q_{\phi} (z,c|x)}{\bigg[\log \Big[\frac{p_{\theta}(x|z,c)p_{\theta}(z|c)p_{\theta}(c)}{q_{\phi} (z,c|x)}\Big] \bigg]} &&&\text{from Jensen's inequality}
    \end{align*} 

    Then with $q_{\phi} (z,c|x) = q_{\phi} (c|x)q_{\phi} (z|c,x)$ we get:
    \begin{align}
    &\log p_{\theta} (x)  &&\geq \mathbb{E}_{z,c \sim q_{\phi} (z,c|x)}{\bigg[\log p_{\theta}(x|z,c) + \log \frac{p_{\theta}(z|c)}{q_{\phi}(z|c,x)} + \log \frac{p_{\theta}(c)}{q_{\phi}(c,x)} \bigg]}
    \end{align} 

    Finally:

    \begin{equation}
    \begin{multlined}
    \mathcal{L} \geq \mathbb{E}_{p_d(x)} \Big[ \mathbb{E}_{z,c \sim q_{\phi} (z,c|x)}{\big[\log p_{\theta}(x|z,c)}\big] + \mathbb{E}_{c \sim q_{\phi}{(c | x)}}{\big[KL \big(q_\phi(z|c, x) \big\| p_{\theta}(z|c) \big)\big]} \\+ KL\big(q_\phi(c | x) \big\| p_{\theta}(c) \big) \Big]
    \end{multlined}
    \end{equation}

\section{Architecture for MNIST and FashionMNIST experiments}
\label{app1}

For all experiments, the Adam optimizer is used with a learning rate of 1e-4.
We chose $\lambda=2, \beta_{cont}=10, \beta_{cat}=10$ and $\beta_{info}=100$

\begin{table}[h!]
  \begin{center}
    \begin{tabular}{l|l}
      \textbf{discriminator D / encoder Q} & \textbf{decoder G}\\
      \hline
      {Input $batch \times 784$ flattened Gray image} & Input $\in \mathbb{R}^{20}$ \\
      FC. $784 \times 400$ + Dropout(0.25) + ReLU & FC. $20 \times 400$ + Dropout(0.25) + ReLU \\
      \textbf{D}: FC. $400 \times 10$ + Softmax \\ $\textbf{Q}_{\mu}$: FC. $410 \times 20$ / $\textbf{Q}_{\sigma}$: FC. $410 \times 20$ & FC. $400 \times 784$ + Dropout(0.25) + Sigmoid \\
    \end{tabular}
  \end{center}
  \caption{Architecture used for MNIST and Fashion MNIST experiments. $\textbf{Q}_{\mu}$ and $\textbf{Q}_{\sigma}$ are the FC nets for respectively the inference of mean and log-variance of the conditional posterior distribution. The dropout is inserted only to avoid over-fitting.}
  \label{tab:table1}
\end{table}

\section{Architecture for multivariate sequential MNIST experiments}
\label{app3}

For all experiments, the Adam optimizer is used with a learning rate of 1e-4.
We chose $\lambda=2, \beta_{cont}=10, \beta_{cat}=10$ and $\beta_{info}=100$

\begin{table}[h!]
  \begin{center}
    \begin{tabular}{l|l}
      \textbf{discriminator D / encoder Q} & \textbf{decoder G}\\
      \hline
      {Input $28 \times batch \times 28$} & Input $\in \mathbb{R}^{20}$ \\
      bidirectional GRU. $28 \times 128$ & FC. $20 \times 128$ \\
      \textbf{D}: FC. $256 \times 10$ + Softmax & GRU. $20 \times 128$ + ReLU \\
      $\textbf{Q}_{\mu}$: FC. $266 \times 20$ / $\textbf{Q}_{\sigma}$: FC. $266 \times 20$ & FC. $128 \times 28$ + Sigmoid \\
    \end{tabular}
  \end{center}
  \caption{Architecture used for 28-dimensional sequential MNIST experiment. $\textbf{Q}_{\mu}$ and $\textbf{Q}_{\sigma}$ are the FC nets for respectively the inference of mean and log-variance of the conditional posterior distribution}
  \label{tab:table4}
\end{table}

The 28-dimensional MNIST sequences are considered as real-valued times series. Therefore, the loss function is the sum of the mean squared error of each time step.

\end{document}